\documentclass[wcp]{jmlr}

\usepackage{longtable}% for long tables

\usepackage{booktabs}
\usepackage[T1]{fontenc}
\usepackage{graphicx}
\usepackage{algorithm}
\usepackage{algorithmic}
\usepackage{times}
\usepackage{soul}
\usepackage{url}
\usepackage{graphicx}
\usepackage{amsmath}
\usepackage{booktabs}
\usepackage{hyperref}

\usepackage{amsfonts}       % blackboard math symbols
\usepackage{nicefrac}       % compact symbols for 1/2, etc.
\usepackage{microtype}      % microtypography
\usepackage{adjustbox}
\usepackage{array}
\usepackage{xcolor}
\usepackage{blindtext}
\usepackage{mathtools}
\usepackage{enumitem}

\renewcommand{\tilde}{\widetilde}

\makeatletter
\let\Ginclude@graphics\@org@Ginclude@graphics 
\makeatother

\jmlrvolume{}
\jmlryear{2022}
\jmlrworkshop{ACML 2022}

\title[NLDP PAC Learning Halfspaces]{On PAC Learning Halfspaces in Non-interactive Local Privacy Model with Public Unlabeled Data}

  \author{\Name{Jinyan Su} \Email{jinyan.su@mbzuai.ac.ae}\\
  \addr Mohamed bin Zayed University of Artificial Intelligence (MBZUAI)  \\
\Name{Jinhui Xu} \Email{jinhui@buffalo.edu}\\
  \addr Department of Computer Science and Engineering\\ 
  State University of New York at Buffalo  \\ 
\Name{Di Wang} \Email{di.wang@kaust.edu.sa}\\
  \addr Division of CEMSE\\
  Computational Bioscience Research Center \\ 
  SDAIA-KAUST Center of Excellence in Data Science and Artificial Intelligence\\ 
  King Abdullah University of Science and Technology (KAUST)\\
 }
\editors{Emtiyaz Khan and Mehmet Gonen}

\begin{document}

\maketitle

\begin{abstract}
In this paper, we study the problem of PAC learning halfspaces in the non-interactive local differential privacy model (NLDP).  To breach the barrier of exponential sample complexity, previous results studied a relaxed setting where the server has access to some additional public but unlabeled data. We continue in this direction. Specifically, we consider the problem under the standard setting instead of the large margin setting studied before. Under different mild assumptions on the underlying data distribution, we propose two approaches that are based on the Massart noise model and self-supervised learning and show that it is possible to achieve sample complexities that are only linear in the dimension and polynomial in other terms for both private and public data, which significantly improve the previous results. Our methods could also be used for other private PAC learning problems. \footnote{Part of the work was done when Jinyan Su was a research intern at KAUST.} 
\end{abstract}

\section{Introduction}
A tremendous quantity of sensitive data is generated and gathered every day. Due to the sensitive information of these data, how to enable the benefit of analyzing the data without exposing the individual information has become an important issue. To address the issue, \textit{Differential Privacy} (DP) \cite{dwork2006calibrating} has become as the de facto tool for privacy-preserving data analysis. 
%which applies randomized response to defend adversaries so that change of a single data could not be detected by the adversary.
There are two well-studied models in DP- the \textit{central} model and the \textit{local} model. In the central model, the raw data is collected by a central server and then processed by a DP algorithm while in the local model \cite{evfimievski2003limiting}, each individual applies a DP algorithm locally and sends only the output of the algorithm to the server. Local model is used more often when learning in a distributed system or when users do not trust the central data collector. 

In the local differential privacy (LDP) model, 
%each individual manages his/her data and disclose them to the server through some DP mechanisms. The server then collects the private data and analyzes them. Namely, the server requests the user to locally run a differentially private algorithm on the basis of their previous communications and the users provide the outputs of the algorithm rather than the raw data.
the communication between the server and individual users could be either in one round or in multiple rounds, and these two communication protocols of LDP are called non-interactive LDP (NLDP) or interactive LDP correspondingly. However, in practice, NLDP is preferred over interactive LDP because of the latency and waiting for responses takes a large amount of time, and thus it is necessary to limit the number of interactions. Moreover, current deployments of LDP algorithms are all non-interactive protocols, such as Google and Apple \cite{cormode2018privacy,tang2017privacy,erlingsson2014rappor,near2018differential}.

Beginning from \cite{kasiviswanathan2011can}, there is a long list of work studying the Valiant’s probabilistically approximately correct (PAC) learning model \cite{valiant1984theory} under DP constraint and what concepts we can learn privately, such as \cite{blum2013learning,bun2020equivalence}. While private PAC learning is well studied in the central DP model and interactive LDP model, its theoretical behaviors in the NLDP  model are much more challenging and are still far from well-understood due to the restriction on the number of rounds of communication. 
\cite{daniely2019locally} 
provided the first study of the problem and proved that only classes that have polynomially small margin complexity can be efficiently PAC
learned by an NLDP algorithm. Recently, \cite{dagan2020interaction} studied the PAC learning halfspaces in NLDP model. While halfspaces is PAC learnable in the central DP model and interactive LDP model \cite{le2020efficient,beimel2019private,kasiviswanathan2011can}, unfortunately, \cite{dagan2020interaction} showed that even for learning halfspaces under large-margin assumptions requires an exponential
number of samples in the NLDP model, which indicates that in general, halfspaces is unlearnable in NLDP model. 
To breach the barrier of exponential sample complexity,  \cite{daniely2019locally} studied a relaxed NLDP model where the server is allowed to access some public but unlabeled data. Specifically,  they considered the large margin setting and showed the following result (see Section \ref{sec:pre} for the definitions of large margin setting and NLDP Learner).\footnote{Since in \cite{daniely2019locally}  did not provide the explicit form of the sample complexities, in Theorem \ref{th2} we rewrite their result, see Appendix for its proof. } 
\begin{theorem}\cite{daniely2019locally}\label{th2}
Under the large margin setting, there is a  computationally efficient  $(\epsilon, \alpha, \beta, \gamma)$-NLDP Learner with sample complexity $n=\tilde{O}(\frac{d^{10}\log(1/\beta)}{\epsilon^2\cdot\gamma^{12}\alpha^6})$ for  private data and $m=\tilde{O}(\frac{d^{10}\log(1/\beta)}{\epsilon^2\cdot\gamma^{12}\alpha^6})$ for public unlabeled data, where $d$ is the dimension of the space, $\gamma$ is the margin, $\alpha$ is the target error and $\beta$ is the failure probability.
\end{theorem} 
However, there are two issues with the result. First, Theorem \ref{th2} only holds for the large margin setting, which is stronger than the standard (non-large margin) setting. Secondly, compared with the non-private case where the  sample complexity is only linear on $\frac{1}{\gamma}$ and is independent on $d$ \cite{shalev2014understanding}, the dependencies on $d, \frac{1}{\gamma}$ in Theorem \ref{th2} are
unsatisfactory. In this paper, we revisit the problem and partially address these issues. Specifically, we consider PAC learning halfspaces in the NLDP model under the standard setting and show that it is possible to achieve a sample complexity that is only {\bf linear} in $d$ (and polynomial in other terms) for both private and public data, if the underlying data distribution satisfies some mild assumptions.
{Our contributions can be summarized as follows.}

1. We first study the case where the data distribution satisfies the anti-anti-concentration and anti-concentration properties. We propose an $(\epsilon, \delta)$-NLDP algorithm which is motivated by the Massart noise learning model and show that its sample complexity to achieve the error $\alpha$ is $\tilde{O}(d\text{Poly}(\frac{1}{\epsilon}, \frac{1}{\alpha}))$ and $O(\frac{d}{\alpha^4})$ for private and public data respectively.

2. To further reduce the sample complexity of public data, we then study the case where the underlying distribution follow a mixture distribution and show that it is possible to achieve sample complexity of  $\tilde{O}(d\text{Poly}(\frac{1}{\epsilon}, \frac{1}{\alpha}))$ and $\tilde{O}(\frac{d}{\alpha^2})$ for private and public data respectively. Instead of the Massart noise model, our algorithm is motivated by self-supervised learning.

Due to the space limit, all proofs and omitted algorithms are included in Appendix.

\section{Related Work} 
\begin{table*}[t]\Huge
\begin{center}
\resizebox{\textwidth}{!}{%	
\begin{tabular}{|l|l|l|l|l|l|}
\hline
Methods                 & Sample Complexity                                                                                                                                                          & Measure           & Loss Function                       & With public data?                & Data                  \\ [3ex] \hline

\cite{smith2017interaction} &  $O(d\epsilon^{-2}\alpha^{-2})$  & Excess Risk & Linear Regression & No & $\ell_2$-norm Bounded  \\[3ex]
					\hline
\cite{smith2017interaction}    & $\tilde{O}(2^d\alpha^{-(d+1)}\epsilon^{-2})$                                                                                                                               & Excess Risk       & Lipschitz and Convex                & No                  & $\ell_2$-norm Bounded \\ [3ex]\hline

                   \cite{wang2018empirical}     & $\tilde{O}(4^{d(d+1)}D^2_d\epsilon^{-2}\alpha^{-4})$                                                                                                                       & Excess Risk       & $(\infty, T)$-smooth                & No                  & $\ell_2$-norm Bounded \\ [3ex]\hline
        \cite{wang2019noninteractive,wang2020empirical}               & $d\cdot \left(\frac{C}{\alpha^3}\right)^{O( 1/\alpha^3)}/\epsilon^{O(\frac{1}{\alpha^3})}$                                                                                 & Excess Risk       & Lipschitz Convex GLM & No                  & $\ell_2$-norm Bounded \\[3ex] \hline
\cite{zheng2017collect} & \begin{tabular}[c]{@{}l@{}}$d(	\frac{8}{\alpha})^{O(\log\log(\frac{1}{\alpha}))}(\frac{4}{\epsilon})^{O(\log(\frac{1}{\alpha}))}$\end{tabular} & Excess Risk       & Convex $\infty$-Smooth GLM           & No                & $\ell_2$-norm Bounded \\[3ex] \hline
\cite{Wang2021GeneralizedLM}            & $O(d^3\alpha^{-2}\epsilon^{-2})$                                                                                                                                                        & $\ell_2$-norm Error     & Smooth GLM        & Yes, $O(\frac{d}{\alpha^2})$ & Gaussian              \\[3ex] \hline
\cite{Wang2021GeneralizedLM}               & \begin{tabular}[c]{@{}l@{}} $O(d^2\alpha^{-2}\epsilon^{-2})$\\ for $\alpha\geq \Omega(\frac{1}{\sqrt{d}})$ \end{tabular}                                                                                                                                                        & $\ell_\infty$-norm  Error & Smooth GLM        & Yes, $O(\frac{d}{\alpha^2})$ &  \begin{tabular}[c]{@{}l@{}}$\ell_1$-norm Bounded \\ and Sub-Gaussian\end{tabular} \\ [3ex]\hline

 \cite{daniely2019locally}             & $\tilde{O}(\frac{d^{10}}{\epsilon^2\cdot\gamma^{12}\alpha^6})$                                                                                                                                                        & Excess Risk    & 0-1 loss/large margin halfspace     & Yes, $\tilde{O}(\frac{d^{10}}{\epsilon^2\cdot\gamma^{12}\alpha^6})$ & $\ell_2$-norm Bounded              \\[3ex] \hline 
{\bf This Paper}            & $\tilde{O}(d\text{Poly}(\frac{1}{\epsilon}, \frac{1}{\alpha}))$                                                                                                                                                        & Excess Risk    & 0-1 loss/ halfspace     & Yes, $O(\frac{d}{\alpha^4})$ & Structured distribution               \\[3ex] \hline
{\bf This Paper}            & $\tilde{O}(d\text{Poly}(\frac{1}{\epsilon}, \frac{1}{\alpha}))$                                                                                                                                                        & Excess Risk    & 0-1 loss/halfspace       & Yes, $\tilde{O}(\frac{d}{\alpha^2})$ & Structured distribution               \\[3ex] \hline
\end{tabular}}
\caption{Comparisons on the sample complexities for private and public unlabeled data to achieve error $\alpha$ under different measurements and assumptions,   where $C$ is a constant and $D_d$ is a function of $d$, $\gamma$ is the margin of the data. For bounded norm case we assume that $\|x_i\| \leq R=O(1)$  for every $i\in [n]$. We also assume the loss functions are Lipschitz.}
\label{Table:1}
\end{center}
\end{table*}
As mentioned, although there are numerous results on private PAC learning halfspaces, 
the problem in the NLDP model with public unlabeled data has
only been studied by \cite{daniely2019locally}.  However, it differs from our results in quite a few ways. Firstly, their algorithm considered a large margin setting and can not be applied to the general setting which is studied in this paper. 
%Secondly, their methods rely on the statistical query model to characterize the LDP protocol with the distribution independent setting while we train a learner via either Massart noise model or self-supervised learning depending on different distributions.
Secondly, although both their work and ours used public unlabeled data, the usage of these data is quite different. \cite{daniely2019locally} used the public unlabeled data to compute the gradient of the decomposed gradient while we use algorithms to label the public data and conduct the learning process on the public data. Finally, \cite{daniely2019locally} studied Data-Independent PAC learning while we focus on Data-Dependent PAC learning. Thus, our sample complexities are lower than theirs. 

Besides PAC learning, recently there are several works studied the problem of Stochastic Convex Optimization in NLDP model (without public data), such as  \cite{smith2017interaction,wang2018empirical,wang2019noninteractive,wang2020empirical,zheng2017collect}, see Table \ref{Table:1} for a summary. However, as we can see from Table \ref{Table:1}, all of these methods need to assume the loss function is smooth enough and the sample complexities of these methods are exponential in $d$ or the error $\alpha$. Thus, these methods cannot be used for our problem as our loss is $0-1$ loss and we aim to make the sample complexity to be polynomial. To remedy the exponential sample complexity, \cite{Wang2021GeneralizedLM} studied the Generalized Linear Model in NLDP model with public unlabeled data. However, they need to assume that the loss function is smooth and the polynomial sample complexity only holds when the error $\alpha$ is not small enough. While our results hold for any $\alpha\in (0, 1)$. Moreover, the usage of the public data is quite different.

\section{Preliminaries}\label{sec:pre}

In this section, we will introduce some notations in PAC learning halfspaces and differential privacy.

\noindent \textbf{Notations: }
Throughout the whole paper, we denote $\mathcal{P}$ as a probability distribution over $\mathcal{X} \times \{\pm 1\}$ with marginal distribution $\mathcal{P}_x$ over $\mathbb{R}^d$, where $\mathcal{X} \subseteq \mathbb{R}^d$. We also denote $\mathcal{B}_2^d(R)$ as the $\ell_2$-norm ball in $\mathbb{R}^d$ with center $0$ and radius $R$ and $\mathcal{B}_2^d=\mathcal{B}_2^d(1)$ as the unit $\ell_2$-norm ball. Given  a convex  constraint set $\mathcal{C}\subseteq \mathbb{R}^d$ and a loss function $\ell: \mathcal{C}\times (\mathcal{X} \times \{\pm 1\})$, we denote
the population risk function as $L_{\mathcal{P}}(w)=\mathbb{E}_{(x,y)\sim \mathcal{P}}[\ell(w; x, y)]$. Moreover, given an $n$-size dataset $D=\{(x_1, y_1), \cdots, (x_n, y_n)\} \sim \mathcal{P}^n$ we denote the empirical risk function of the loss over $D$, $\hat{L}(\cdot, D)$, as $\hat{L}(w,D)=\frac{1}{n}\sum\limits_{i=1}^n\ell(w; x_i,y_i)$.
\subsection{PAC Learning Halfspaces}
In this paper we mainly focus on PAC (probably approximately correct) learning model \cite{valiant1984theory}  for halfspaces in the realizable setting. That is, for any sample $(x,y)\sim \mathcal{P}$  we assume that $y=\text{sign}(\langle w^{*}, x\rangle + \theta^*)$ (almost surely) for some unknown vector $w^{*}\in \mathbb{R}^d$ and $\theta^*\in \mathbb{R}$. Without loss of generality we assume that $\theta^*=0$ so $y=\text{sign}(\langle w^{*}, x\rangle)$. 
 A linear threshold function is defined as $f_w(x)=\text{sign}(\langle w,x\rangle)$ where $x,w\in \mathbb{R}^d$ and we call the vector $w$ a hypothesis, and the classification error of hypothesis $w$ is 
\begin{align*}
    err_{\mathcal{P}}(f_w)&=\underset{(x,y)\sim \mathcal{P}}{Pr}[f_w(x)\neq y]=\underset{(x,y)\sim \mathcal{P}}{Pr}[\text{sign}(\langle w,x\rangle)\neq y]\\&=\underset{(x,y)\sim \mathcal{P}}{Pr}[y\cdot \langle w,x\rangle<0].
\end{align*}
Given $\alpha,\beta \in (0, 1)$, the goal of  PAC learning halfspaces is to find a hypothesis $w\in \mathbb{R}^d$ such that $err_{\mathcal{P}}(f_w)\leq \alpha$ with probability at least $1-\beta$ with low sample complexity. In the following we will introduce both the standard setting and the large margin setting. 
\paragraph{{\bf Standard setting:}} Here we assume without loss of generality that  $\mathcal{X}\subset 
\mathcal{B}^d_2(R)$ with some constant $R=O(1)$ and $w^*\in \mathbb{R}^d_2$.  Formally, we aim to design an $(\alpha, \beta)$-PAC learner.
\begin{definition}[$(\alpha, \beta)$-PAC learner]
    Let $\mathcal{P}$ be a distribution over $\mathcal{B}_2^d(R) \times \{\pm 1\}$ such that there exists $w^{*}\in \mathcal{B}_2^d$ which satisfies $Pr_{(x,y)\sim \mathcal{P}}[y\langle w^{*},x\rangle\geq 0]=1$. We say an algorithm $\mathcal{A}$ an $(\alpha,\beta)$-PAC learner with  sample complexity $n$ if using a dataset $D\sim \mathcal{P}^n$, the output classifier $\hat{w}=\mathcal{A}(D)\in \mathcal{B}_2^d$ satisfies $\underset{(x,y)\sim \mathcal{P}}{Pr}[y\neq \text{sign}(\langle \hat{w},x\rangle)]\leq \alpha $  with probability at least $1-\beta$. 

\end{definition}
\paragraph{{\bf Large margin setting:}} Compared with the standard setting, in the large margin setting we additionally assume  there is no example that falls too close to the boundary of the halfspace. Specifically, we assume  that  $\mathcal{X}\subset 
\mathcal{B}^d_2(R)$ with some constant $R=O(1)$ and $w^*\in \mathbb{R}^d_2$.  Moreover, we assume that $w^*$ maximizes the margin 
 $\gamma=\min_{(x,y)\sim \mathcal{P}}\frac{|\langle w^{*},x\rangle|}{||w^{*}||_2\cdot ||x||_2}>0$, which is known in advance.  Under this setting we want to design an $(\alpha, \beta, \gamma)$-PAC learner. 
%When margin $\gamma>0$, it is possible to design a PAC learning algorithm (private and non-private) for learning halfspaces requiring sample complexity to grow polynomially in the known parameter $\frac{1}{\gamma}$ and independently with dimension $d$. More specifically, when no example falls too close to the boundary of the halfspace, the goal of PAC learning model will become designing an $(\alpha, \beta, \gamma)$-PAC learner to learn this large margin halfspace.
\begin{definition}
[$(\alpha,\beta,\gamma)$-PAC learner]
Let $\mathcal{P}$ be a distribution over $\mathcal{B}_2^d(R) \times \{\pm 1\}$ such that there exists $w^{*}\in \mathcal{B}_2^d$ which satisfies $Pr_{(x,y)\sim \mathcal{P}}[y\langle w^{*},x\rangle\geq \gamma]=1$, then we call $\mathcal{P}$ a distribution with margin $\gamma$. We say an algorithm $\mathcal{A}$ an $(\alpha,\beta,\gamma)$-PAC learner with margin $\gamma$ and sample complexity $n$ if using a dataset $D\sim \mathcal{P}^n$ with margin $\gamma$, the output classifier $\hat{w}=\mathcal{A}(D)\in \mathcal{B}_2^d$ satisfies the     $\underset{(x,y)\sim \mathcal{P}}{Pr}[y\neq \text{sign}(\langle \hat{w},x\rangle)]\leq \alpha$ with probability at least $1-\beta$.
\end{definition}

\subsection{Differential Privacy}
%To design algorithms that ensure privacy, one of the standard approaches is using differential privacy. Differential privacy requires that no individual samples in the input influence the learned classifier significantly. The formal definition of differential privacy is given below:
\begin{definition}[Differential Privacy \cite{dwork2006calibrating}]
Given a data universe $\mathcal{D}$, we say that two datasets $D, D'\subseteq 
\mathcal{D}^n$ are neighbors if they differ by only one entry, which is
denoted as $D\sim D'$. A randomized algorithm $\mathcal{A}$ is $(\epsilon,\delta)$-differentially private (DP) if for all neighboring datasets $D, D'$ and all output event $E$ of algorithm $\mathcal{A}$, we have 
	$\text{Pr}(\mathcal{A}(D)\in E)\leq e^{\epsilon}\text{Pr}(\mathcal{A}(D')\in E)+\delta. $
	If $\delta=0$, we say that algorithm $\mathcal{A}$ is $\epsilon$-DP. 
\end{definition}
\paragraph{Differential privacy in the local model.} In LDP, we have a data universe $\mathcal{D}$,  $n$ players with each holding  a private data record $x_i\in \mathcal{D}$, and a server 
%that is in charge of 
coordinating the protocol. An LDP protocol executes a total of $T$ rounds. In each round, the server sends a message, which is also called a query, to a subset of the players requesting them to run a particular algorithm. Based on the query, each player $i$ in the subset selects an algorithm $\mathcal{A}_i$, runs it on her own data, and sends the output back to the server.

In NLDP model,  we consider the distributed setting with star network. And each user has only
one data sample. He/she needs to privatize his/her message before sending to the sever, and then the
server aggregate these private information to perform analysis. Unlike the federated setting,
since each user only has one sample, he/she cannot compute the target
locally. And this is the main difficulty
of learning in the NLDP model.
\begin{definition}[\cite{dwork2006calibrating}]\label{definition:2}
An algorithm $\mathcal{A}$ is $(\epsilon, \delta)$-locally differentially private (LDP) if for all pairs $x,x'\in \mathcal{D}$, and for all events $E$ in the output space of $\mathcal{A}$, we have
    $\text{Pr}[\mathcal{A}(x)\in E]\leq e^{\epsilon}\text{Pr}[\mathcal{A}(x')\in E]+\delta.$
A multi-player protocol is $(\epsilon, \delta)$-LDP if for all possible inputs and runs of the protocol, the transcript of player i's interaction with the server is $(\epsilon, \delta)$-LDP. If $T=1$, we say that the protocol is $\epsilon$ non-interactive LDP (NLDP). When $\delta=0$, we call it  $\epsilon$-NLDP.
\end{definition} 
%A useful property of (local) differential privacy is that it is closed under post-processing. \begin{lemma}  Let $\mathcal{A}: \mathcal{D}^n \mapsto \mathcal{O}$ be a randomized algorithm that is $(\epsilon, \delta)$-(locally) differentially private. Then for $f: \mathcal{O}\mapsto \mathcal{O}'$, $f(\mathcal{A}(\cdot))$ is $(\epsilon, \delta)$-(local) DP. \end{lemma}
As we mentioned previously, PAC learning halfspaces in the NLDP model requires the sample complexity which is at least exponential in the dimension $d$ even in the large margin setting \cite{dagan2020interaction}. Thus, inspired by this, instead of the NLDP model, in this paper we will mainly focus on a relaxed NLDP model. 
\paragraph{Our Model:} Different from the above classical NLDP  model  where only one private dataset $D=\{(x_i, y_i)\}_{i=1}^n$ exists, the NLDP model in our setting allows
%we assume that 
the server to have an additional public unlabeled dataset $D'=\{q_j\}_{j=1}^{m}\subset \mathcal{X}^m$, where each $q_j$ is sampled from  $\mathcal{P}_x$, which %$\mathcal{P}_x$ 
is the marginal distribution of $\mathcal{P}$. 

Thus, we aim to design some private $(\alpha, \beta)$ or $(\alpha, \beta, \gamma)$-PAC learner in the NLDP model with public but unlabeled data. Moreover, we want the sample complexity of private data and public data to be as low as possible. 
\begin{definition}[$(\epsilon, \delta, \alpha, \beta)$-NLDP Learner]
    Let $\mathcal{P}$ be a distribution over $\mathcal{B}_2^d(R) \times \{\pm 1\}$ such that there exists $w^{*}\in \mathcal{B}_2^d$ which satisfies $Pr_{(x,y)\sim \mathcal{P}}[y\langle w^{*},x\rangle\geq 0]=1$. We call an algorithm $\mathcal{A}$ an $(\epsilon, \delta, \alpha, \beta)$-NLDP PAC learner with  sample complexity $(n, m)$  if using a (private) dataset $D\sim \mathcal{P}^n$ and a public but unlabeled dataset $D'\sim \mathcal{P}_x^m$, the output classifier $\hat{w}=\mathcal{A}(D, D')\in \mathcal{B}_2^d$ satisfies the following with probability at least $1-\beta$, 
    $\underset{(x,y)\sim \mathcal{P}}{Pr}[y\neq \text{sign}(\langle \hat{w},x\rangle)]\leq \alpha.$
Moreover, the algorithm $\mathcal{A}$ is $(\epsilon, \delta)$-NLDP w.r.t the private dataset.  
\end{definition}
\begin{definition}[$(\epsilon, \delta, \alpha, \beta, \gamma)$-NLDP Learner]
      Let $\mathcal{P}$ be a distribution over $\mathcal{B}_2^d(R) \times \{\pm 1\}$ such that there exists $w^{*}\in \mathcal{B}_2^d$ which satisfies $Pr_{(x,y)\sim \mathcal{P}}[y\langle w^{*},x\rangle\geq \gamma ]=1$. We call an algorithm $\mathcal{A}$ an $(\epsilon, \delta, \alpha, \beta)$-NLDP PAC learner with  sample complexity $(n, m)$  if using a (private) dataset $D\sim \mathcal{P}^n$ and a public but unlabeled dataset $D'\sim \mathcal{P}_x^m$, the output classifier $\hat{w}=\mathcal{A}(D, D')\in \mathcal{B}_2^d$ satisfies the following with probability at least $1-\beta$, 
    $\underset{(x,y)\sim \mathcal{P}}{Pr}[y\neq \text{sign}(\langle \hat{w},x\rangle)]\leq \alpha. $
Moreover, the algorithm $\mathcal{A}$ is $(\epsilon, \delta)$-NLDP w.r.t the private dataset.   
\end{definition}
Since  any $(\epsilon, \delta)$-NLDP algorithm can be transformed to an $\epsilon$-NLDP algorithm with almost the same accuracy \cite{bun2019heavy}, here we only focus on $(\epsilon, \delta)$-NLDP for simplicity. 

\section{NLDP Algorithm via Massart noise model}\label{se1}
Before showing our algorithm, we first introduce the Massart noise model in PAC learning:
\begin{definition}[Massart noise example oracle \cite{massart2006risk}] 
	Let $\mathcal{S}$ be a concept class of Boolean functions over $\mathbb{R}^d$, $\mathcal{F}$ be a known family of structured distributions on $\mathbb{R}^d$, and let $f$ be an unknown target function in $\mathcal{S}$.  Assume $0<\lambda<\frac{1}{2}$, 
	a Massart noise example oracle $EX^{Mas}(f,\mathcal{F},\lambda)$ is an oracle that each invoke returns a labeled example $(x,y)$ such that:
	\begin{enumerate}
		\item $x\sim \mathcal{P}_x$, where $\mathcal{P}_x$ is  a fixed distribution in $\mathcal{F}$.
		\item With probability $1-\lambda(x)$, the oracle returns the correct label $y=f(x)$ and with probability $\lambda(x)$, the oracle returns a misleading label $y=-f(x)$,
		i.e.,
		$y=\left\{\begin{aligned}
	f(x), &~ \text{w.p.} ~1-\lambda(x)\\ 
		-f(x), &~ \text{w.p.} ~\lambda(x)
		\end{aligned}
		\right.$, where $\lambda(x)$ is unknown and satisfies $\lambda(x)\leq \lambda$.
	\end{enumerate}
\end{definition}
		\begin{algorithm}[!htbp]
	\caption{NLDP based on Massart noise model \label{alg1}}
	\begin{algorithmic}[1]
	\STATE {\bfseries Input:}  	
Private data $D=\{(x_i,y_i)\}_{i=1}^n$ with each $x_i \in \mathbb{R}^d$ satisfying $||x_i||_2\leq R$ and $y_i \in \{\pm 1\}$; Unlabeled public data $D'=\{q_i\}_{i=1}^{m}$; private parameters $\epsilon,\delta$; oracle access to Hinge Loss-LDP  $\mathcal{H}_{priv}$ (Algorithm \ref{alg11}); error bound $\alpha$; failure probability $\beta$. 
\STATE Randomly divide $n$ private data record into $k$ groups: $\{S_1,\cdots, S_k\}$, where $|S_i|=\lfloor \frac{n}{k}\rfloor$, $k=O(\log \frac{1}{\beta})$.
\FOR {$t\in [k]$}
    \STATE Denote $\tilde{S}_t$ as the normalized version of $S_t$, {\em i.e.,} $\tilde{S}_t=\{(\frac{x}{R},y )| (x,y)\in S_t\}$. 
	\STATE Set $w_t=\mathcal{H}_{priv}(\frac{1}{32R}, \epsilon, \delta, \tilde{S}_t)$.

	\STATE Set $h_{w_t}(x)=\text{sign}(w_t^T x)$
		\ENDFOR
	\STATE Get the Massart Noise example oracle by majority voting:
	$\hat{f}(x)=\arg\min_{y\in\{\pm 1\}}\sum\limits_{t=1}^{k}\mathbb{I}(h_{w_t}(x)\neq y)=\arg\min_{y\in\{\pm 1\}}\sum\limits_{t=1}^{k}\mathbb{I}(\text{sign}(w_t^T x)\neq y)$
\FOR {$i \in [m]$ }
	\STATE Label public dataset $\{q_i\}_{i=1}^m$ using Massart Noise example oracle $\hat{f}(x)$ to obtain the noisy dataset $\hat{D}=\{(q_i,\hat{f_i})\}_{i=1}^{m}$, where $\hat{f_{i}}=\hat{f}(q_i)$.
	\ENDFOR
	\STATE Run the subroutine LHMN($\alpha,\beta,(U,r, R)$) with dataset $\hat{D}$ to get $\hat{w}=$LHMN$(\alpha,\beta,(U,r, R))$.\\
	\STATE Return $\hat{w}$ and $h_{\hat{w}}=\text{sign}(\hat{w}^Tx)$.
		\end{algorithmic}
\end{algorithm}
We can think of the Massart noise model as an adversary who flips each sample label independently with probability \textbf{at most} $\lambda<\frac{1}{2}$ and the target of PAC learner is to reconstruct the classifier to arbitrarily high accuracy. 
\begin{definition}[PAC Learning with Massart Noise]
  Denote	$\mathcal{P}$ the joint distribution on $(x,y)$ generated by a Massart noise example oracle.  The goal of PAC Learning with Massart Noise is to design an algorithm such that given i.i.d. samples from $\mathcal{P}$, the algorithm outputs a hypothesis $h$ such that $Pr_{x\sim \mathcal{P}_x}[h(x)\neq f(x)]\leq \alpha$ with probability at least $1-\beta$. 
\end{definition}
Massart noise model lies in between the Random Classification Noise \cite{angluin1988learning} (where each label is independently flipped with probability \textbf{exactly} $\lambda\leq \frac{1}{2}$) and the agnostic model \cite{kearns1994toward} (where an adversary can flip any small constant fraction of the sample labels) and has attracted much attention in recent years. Many algorithms for computing accurate hypothesis in the distribution-specific PAC learning has been promoted, such as \cite{awasthi2015efficient,awasthi2016learning,zhang2017hitting}. Recently, an efficient and simple algorithm has been proposed in \cite{diakonikolas2020learning}, which succeeds under more general distributional assumptions.

\begin{definition}[\cite{diakonikolas2020learning}]\label{de1}
	Fix $U,r>0$. An isotropic (i.e., zero mean and identity covariance) distribution $\mathcal{P}_x$ on $\mathbb{R}^d$ satisfies  $U$-anti-concentration (2-dim) if for any projection $(\mathcal{P}_x)_V$ of $\mathcal{P}_x$ onto a 2-dimensional subspace $V$, the corresponding probability density function $\gamma_V$ on $\mathbb{R}^2$ satisfies that for all $x\in V$,  $\gamma_V(x)\leq U$. Moreover, we say $(U,r)$-anti-anti-concentration holds if for all $x\in V$ such that $||x||_2\leq r$, $\gamma_V(x)\geq \frac{1}{U}$.
	\end{definition}
	Anti-anti-concentration and anti-concentration are mild distributional conditions about the probability density function on the projected 2-dimensional subspace. The former guarantees that at least a constant probability mass is assigned to the points near the origin of the projected 2-dimensional subspace while the latter states that the probability mass along the 2-dimensional projection is upper bounded.

In fact, several reasonable distribution families satisfy the previous two conditions. For example, the class of isotropic log-concave distribution satisfies $(U,r)$-anti-anti-concentration and $U$-anti-concentration with $U,r=\Theta(1)$ (See Fact A.4 in \cite{diakonikolas2020learning}). Moreover, any isotropic s-concave distribution on $\mathbb{R}^d$ with $s\geq -\frac{1}{2d+3}$ satisfies  $(U,r)$-anti-anti-concentration and $U$-anti-concentration with $U,r=\Theta(1)$ (See Appendix A.4 in \cite{diakonikolas2020learning}).

Next we will present our Non-interactive LDP algorithm via the Massart noise model (Algorithm \ref{alg1}). Generally, the algorithm consists of two parts:

(1) First, we use  private data to construct a Massart noise example oracle with rate $\lambda=\frac{3}{16}$. To get the oracle, in Algorithm \ref{alg1} we first 
 randomly divide the private data into $k=O(\log\frac{1}{\beta})$ disjoint groups. Then, on each group of data $S_i$, we  consider the Empirical Risk Minimization problem with the hinge loss $\ell(w; x,y)=\max\{0, 1-y\langle w, x\rangle\}$ with $\mathcal{C}=\mathcal{B}^d_2$. Specifically, when  $n=\tilde{O}\left({kd\log(\frac{1}{\beta})}\text{Poly}(\log \frac{1}{\delta}, \frac{1}{\epsilon})\right)$, we can use an $(\epsilon, \delta)$-NLDP algorithm $\mathcal{H}_{priv}$ given by \cite{wang2020empirical} to get private estimator $w_i$ such that 
 \begin{equation*}
      \mathbb{E}[\hat{L}(w_i,S_i)]- \min_{||w||_2\leq 1}\hat{L}(w,S_i) \leq \frac{1}{32},  
 \end{equation*}
 where $\hat{L}(w,S)=\frac{1}{|S|}\sum_{(x,y)\in S}\ell(w; x,y)$. 
 After getting private estimators $\{w_i\}_{i=1}^k$, we then boost the classification accuracy using the majority voting mechanism. We can show that the new classifier  $\hat{f}$ via  voting is a  Massart noise example oracle ($\lambda=\frac{3}{16}$) with probability at least $1-\beta$.

(2) We then label $m=O(\frac{U^{12}}{r^{12}}  \cdot \frac{d}{\alpha^4})$ public unlabeled data samples $D'=\{q_i\}_{i=1}^m$ with the learned Massart noise example oracle $\hat{f}$ and denote the labels as $\{\hat{f_i}\}_{i=1}^m$, where $U,r$ are the parameters of anti-anti-concentration and anti-concentration in Definition \ref{de1}. Then, we can invoke efficient and non-private algorithm LHMN designed for leaning halfspaces with Massart noise (Algorithm \ref{alg2}) on  dataset $\hat{D}=\{(q_i,\hat{f_i})\}_{i=1}^m$ to finally learn a classifier with any desired classification error $\alpha$ with probability at least $1-\beta$. Formally, Algorithm \ref{alg1} has the following theoretical guarantee.

	\begin{algorithm}
	\caption{Learning Halfspaces with Massart Noise
%	\cite{diakonikolas2020learning}
	: LHMN ($\alpha,\beta,(U,r,R)$)} \label{alg2}
	\textbf{Input:} The designed estimation error $\alpha$; parameters about the distribution: $U,r,R$; failure probability $\beta$; loss function  $g(w;(x,y))=S_{\sigma}(-y\frac{\langle w,x\rangle}{||w||_2})$, where $S_{\sigma}(t)=(1+e^{-\frac{t}{\sigma}})^{-1}$, dataset $\hat{D}=\{(x^{(i)},y^{(i)})\}_{i=1}^m$ labeled by Massart noise example oracle with $\lambda=\frac{3}{16}$.
	\begin{algorithmic}[1]
		\STATE Set $C_1=\Theta(\frac{U^{12}}{r^{12}})$, $C_2=\Theta(\frac{r}{U^2})$, $T=\Theta(\frac{C_1 d R^8\log(\frac{1}{\beta})}{\alpha^4})$,
		 $\sigma=\frac{C_2\alpha}{\sqrt{2}R^2}$
		 \STATE Set step size $\eta=\frac{C_2^2d\alpha^2}{8R^4 T^{1/2}}$.
		\STATE Set $w^{(0)}=\boldsymbol{e_1}$ with $\boldsymbol{e_1}$ being the unit vector whose first component is 1, and other components are 0.

		\FOR {$i=1,\cdots,T$}
		\STATE $v^{(i)}=w^{(i-1)}-\eta \nabla_w g(w^{(i-1)};(x^{(i)},y^{(i)}))$
		\STATE $w^{(i)}=\frac{v^{(i)}}{||v^{(i)}||_2}$
		\ENDFOR
		\STATE Set the list of candidate vector $L=\{\pm w^{(i)} \}_{i\in [T]}$
		\STATE Draw $N=O(\frac{\log(\frac{T}{\beta})}{\alpha^2})$ fresh samples from $\hat{D}$
		\STATE $\bar{w}=\arg\min_{w\in L} \sum\limits_{j=T+1}^{T+N} \mathbb{I}\{\text{sign}(\langle w,x^{(j)}\rangle )\neq y^{(j)} \}$.\\
		\STATE Return $\bar{w}$
	\end{algorithmic}
\end{algorithm}

\begin{theorem}\label{th1}
Let $\mathcal{P}$ be a distribution on $\mathbb{R}^d\times\{\pm 1\}$ such that its marginal distribution $\mathcal{P}_x$ on $\mathbb{R}^d$ satisfies that $(U,r)$-anti-anti-concentration and $U$-anti-concentration with $U, r=\Theta(1)$, and $\|x\|_2\leq R=O(1)$ for $x\sim \mathcal{P}_x$. Then for any $\alpha, \beta, \epsilon, \delta\in (0, 1)$, Algorithm \ref{alg1} is a computationally efficient $(\epsilon, \delta, \alpha, \beta)$-NLDP Learner with sample complexity  $m=O(\frac{d}{\alpha^4})$ for public unlabeled data and $n=\tilde{O}\left({d\log^2(\frac{1}{\beta})}\text{Poly}(\log \frac{1}{\delta}, \frac{1}{\epsilon})\right)$ for private data, where the Big-$\tilde{O}$ omits other logarithmic terms. 
\end{theorem}

\begin{remark}
Firstly, we can see that the sample complexity of private data is independent of the error $\alpha$. This is due to that we only need the private data to construct a Massart noise oracle with $\lambda=O(1)$. Moreover, the classifier $\hat{f}$ is a Massart noise oracle for any distribution as long as $\|x_i\|_2\leq R$ and the assumption of anti-concentration and anti-anti-concentration is only used for Algorithm \ref{alg2}, which indicates that the idea of our algorithm could be used to PAC learning halfspaces with other structured distributions, as long as there is an efficient PAC learning algorithm with Massart noise. 

%Thirdly, here we need to assume the radius $R$ is a constant as the degree of the polynomial in the sample complexity of $n$ depends on $R$ exponentially. 
\end{remark}
\subsection{\bf Proof of Theorem \ref{th1}}
The proof of Theorem \ref{th1} requires the following two lemmas. 
The first lemma suggests that $\hat{f}$ is a Massart Noise example oracle with high probability and the second lemma indicates the performance guarantee of LHMN (Algorithm \ref{alg2}).
\begin{lemma}\label{le7}
Under the standard setting,  for $\beta\in (0,1)$, setting  $k=O(\log(\frac{1}{\beta}))$ in Algorithm \ref{alg1}. Then with sample size $n\geq \tilde{\Omega}({kd\log(\frac{1}{\beta})}\text{Poly}(\log \frac{1}{\delta}, \frac{1}{\epsilon}))$, we have the following with probability at least $1-\beta$, 
	\begin{equation*}
\underset{(x,y)\sim \mathcal{P}}{Pr}[\hat{f}(x)\neq y]\leq \frac{3}{16}.
	\end{equation*}
	\end{lemma}
 
  Lemma \ref{le7} suggests that for any $(x,y)\sim \mathcal{P}$, with probability no more than $\frac{3}{16}$, $\hat{f}(x)$ is adversary and returns the wrong label $\hat{f}(x)=-y$ while with probability at least $\frac{13}{16}$, it returns the correct label $\hat{f}(x)=y$. So, $\hat{f}(x)$ is in fact a Massart noise example oracle with $\lambda=\frac{3}{16}$. Before that we recall the definition of bounded distribution in \cite{diakonikolas2020learning}. 
  \begin{definition}[Bounded Distribution \cite{diakonikolas2020learning}]
      Fix $U, R>0$ and $t: (0, 1)\mapsto \mathbb{R}_+$. An isotropic ({\em i.e.,} zero mean and identity covariance) distribution $\mathcal{P}_x$ on $\mathcal{R}^d$ is called $(U, R, t(\cdot))$-bounded if for any projection $(\mathcal{P}_x)_V$ of $\mathcal{P}_x$ onto a 2-dimensional subspace $V$, the corresponding pdf $\gamma_V$ on $\mathbb{R}^d$ satisfies $(U, R)$-anti-anti-concentration, $U$-anti-concentration and for any $\alpha\in (0, 1)$, $Pr_{x\sim \gamma_V}(\|x\|_2\geq t(\alpha))\leq \alpha$. 
  \end{definition}
 Note that since we assume $\|x\|_2\leq R$. Thus, we always have $t(\alpha)=R$. That is, under the assumption in Theorem \ref{th1}. The marginal distribution $\mathcal{P}_x$ is $(U, r, R)$-bounded.  The next lemma about the performance guarantee of LHMN (Algorithm \ref{alg2}) for $(U, r, R)$-bounded distributions follows directly from Theorem 4.1 in \cite{diakonikolas2020learning} by substituting Massart noise rate with $\lambda=\frac{3}{16}$.
\begin{lemma}\label{le8}
    Let $\mathcal{P}$ be a distribution on $\mathbb{R}^d\times \{\pm 1\}$ such that the marginal distribution $\mathcal{P}_x$ on $\mathbb{R}^d$ is $(U,r,R)$-bounded. Let $\lambda=\frac{
    3}{16}$ be the upper bound on Massart noise rate. Algorithm \ref{alg2} draws $m=O((\frac{U}{r})^{12}\cdot R^8 \cdot \frac{d}{\alpha^4})$ examples labeled by Massart noise example oracle and outputs a hypothesis $\bar{w}$ that satisfies 
    $err_{\mathcal{P}}(h_{\bar{w}})\leq \alpha$ with probability at least $1-\beta$.
\end{lemma}
With the above lemmas, the proof of Theorem \ref{th1} is straight forward.
\begin{proof}[{\bf Proof of Theorem \ref{th1}}]
	According to Lemma \ref{le7} and the definition of Massart noise example oracle, with probability at least $1-\beta$, $\hat{D}=\{(q_i,\hat{f_i})\}_{i=1}^{m}$ can be seen as the data returned by a Massart noise example oracle with $\lambda=\frac{3}{16}$. Then, applying lemma \ref{le8}, it follows directly that $err_{\mathcal{P}}(h_{\hat{w}})\leq \alpha$ with probability at least $1-\beta-\beta =1-2\beta$.
\end{proof}

\section{NLDP Algorithm via Self-supervised Learning}
In the previous section, we showed that if the marginal distribution $\mathcal{P}_x$ satisfies  some mild assumptions, there is an NLDP algorithm using $\tilde{O}(d\text{Poly}(\frac{1}{\epsilon}) )$  private data  and $O(\frac{d}{\alpha^4})$ public unlabeled data to achieve an error of $\alpha$. However, as we mentioned earlier, for smooth Generalized Linear Models with Gaussian data, there is an NLDP algorithm with sample complexity of only $O(\frac{d}{\alpha^2})$ for public data \cite{Wang2021GeneralizedLM}. Thus, our question is, can we further reduce the sample complexity of public data (for other structured distributions)? In this section, we will focus on a class of distributions namely mixture distribution, which is proposed by \cite{frei2021self}. We develop an $(\epsilon, \delta)$-NLDP algorithm that achieves an arbitrary classification error $\alpha$ using only $\tilde{O}(d\text{Poly}( \frac{1}{\epsilon}))$ private data and $O(\frac{d}{\alpha^2})$ public unlabeled data. We begin by introducing the mixture distribution model in  \cite{frei2021self}.
	\begin{algorithm}
	\caption{NLDP for Mixture distributions}\label{alg5}
	\begin{algorithmic}[1]
		\STATE {\bfseries Input:}  Private data $D=\{(x_i,y_i)\}_{i=1}^n$ with each $x_i \in \mathbb{R}^d$ satisfying $||x_i||_2\leq R$ and $y_i \in \{\pm 1\}$; Unlabeled public data $D'=\{q_i\}_{i=1}^{m}$; private parameters $\epsilon,\delta$; oracle access to Logistic Loss-NLDP  $\mathcal{T}_{priv}$ (Algorithm \ref{alg12}); error bound $\alpha$; failure probability $\beta$;
		failure probability $\beta$; privacy parameters $\epsilon$, $\delta$; a constant upper bounded of $\|\mu\|_2$, $\rho$, where $\mu$ is the mean of the distribution of $x$; parameters $r, U$ about distribution of $x$. 
		\STATE Run $\mathcal{T}_{priv}(C_{err}\log 2/2, R, \rho,  \epsilon, \delta, D)$ and denote its output as $w^{priv}$, where $C_{err}=\frac{r^2}{144U}$. 
		\STATE Invoke $\{w^{(t)}\}_{t=0}^T$ =STWN($\{q_i\}_{i=1}^{T\times B},w^{priv})$, where $B=O\left(\frac{\log(\frac{1}{\beta})}{\alpha}\right)$, $T=\tilde{O}\left( \frac{d(\log(\frac{1}{\beta}))^2}{\alpha}  \right)$.  \\
		\STATE Return $\{w^{(t)}\}_{t=0}^T$ 
	\end{algorithmic}
\end{algorithm}

Informally, a mixture distribution model is an isotropic model generating data $(x,y)\in \mathbb{R}^d\times \{\pm 1\}$ as follows: for labels $y\in \{\pm 1\}$ and mean parameter $\boldsymbol{\mu}\in \mathbb{R}^d$, $\boldsymbol{x}|y$ (conditioned on $y$) is a random variable with mean $y\boldsymbol{\mu}$ and identity covariance matrix. Additionally, mixture distribution model requires that $\boldsymbol{x}-y\boldsymbol{\mu}$ to satisfy anti-anti-concentration, anti-concentration (1-dim) and $k$-sub-exponential properties. 
Note that we have already introduced the definitions of anti-anti-concentration and anti-concentration (2-dim) in Definition \ref{de1}. The definition of anti-concentration (1-dim) is almost the same as anti-concentration (2-dim) given in definition \ref{de1}, except substituting the subspace $V$ to a 1-dimensional subspace, which declares that the distribution assigns bounded probability mass along one-dimensional projections. 
\begin{definition}[U-anti-concentration (1-dim)]
Fix $U>0$, we say an isotropic distribution $\mathcal{P}_x$ on $\mathbb{R}^d$ satisfies $U$-anti-concentration (1-dim) if for any projection $(\mathcal{P}_x)_V$ of $\mathcal{P}_x$ into a 1 dimensional subspace $V$ and all $x\in V$, it holds that $\gamma_V(x)\leq U$,
where $\gamma_V$ the probability density function on $\mathbb{R}$.
\end{definition}
%We now give the definition of $k$-sub-exponential distributions, which is satisfied by log-concave isotropic distributions and is a quite standard assumption.
\begin{definition}[K-sub-exponential distributions \cite{frei2021self}] \label{de2} 
We say a distribution $\mathcal{P}_x$	 is $K$-sub-exponential if every $x\sim \mathcal{P}_x$ is a sub-exponential random vector with sub-exponential norm at most $K$. In particular, if for any $\boldsymbol{v}$ with $||\boldsymbol{v}||=1$, $\underset{x\sim \mathcal{P}_x}{Pr}[|\langle \boldsymbol{v},x\rangle|\geq t]\leq e^{-\frac{t}{K}}$, then we say $\mathcal{P}_x$ is $K$-sub-exponential. 
\end{definition}
Now we formally define the mixture distribution model considered in this section.
\begin{definition}[Mixture distribution \cite{frei2021self}]
	Let $\boldsymbol{\mu}\in \mathbb{R}^d$. Let  $y=1$  with the probability $\frac{1}{2}$ and $y=-1$ with probability $\frac{1}{2}$, and we generate $x|y\sim z+y\boldsymbol{\mu}$, where $z$ is an isotropic $K$-sub-exponential distribution  satisfying $(U,r)$-anti-anti-concentration and the $U$-anti-concentration (1-dim), then we say $(x,y)\sim \mathcal{P}$ is a mixture distribution with mean $\boldsymbol{\mu}$ and parameters $K,U, r= \Theta(1)$.
\end{definition}

Log-concave isotropic distributions such as the standard Gaussian are $K$-sub-exponential and satisfy $U$-anti-concentration (1-dim) as well as $(U,r)$-anti-anti-concentration (2-dim) with $K,U, r=\Theta(1)$ \cite{frei2021self}. Thus, the above mixture distribution is a natural generalization of the Gaussian mixture model and can accommodate a broader class of distributions.

Similar to our previous algorithm which is based on the Massart noise model, the main idea of our  NLDP algorithm for mixture distribution also consists of two parts.

(1) We first use an $(\epsilon, \delta)$-NLDP algorithm named Logistic Loss-NLDP (Algorithm \ref{alg12}), which is proposed by \cite{zheng2017collect}), to get a private estimator $w^{priv}$ which could achieve the error at most $C_{err}\log 2/2$ for the expected excess population  risk with logistic loss by using  $\tilde{O}(d\text{Poly}(\frac{1}{\epsilon}, \log \frac{1}{\delta}))$ private data, {\em i.e.,}
\begin{equation*}
    \mathbb{E}[{L}(w^{priv}, D)]- \min_{\|w\|_2\leq \|\boldsymbol{\mu}\|_2} \mathbb{E}[{L}(w, D)]\leq \frac{C_{err} \log 2}{2}, 
\end{equation*}
where  $C_{err}=\frac{r^2}{144 U}>0$,  $U, r$ are parameters of the mixture distribution, and ${L}(w, D)=\mathbb{E}_{(x,y)\sim \mathcal{P}}\ell(y\langle w, x\rangle)$ with $\ell(z)=\log(1+e^{-z})$. Based on this result, we show that $w^{priv}$ could be thought as a pseudo labeler which achieves a sufficiently small but constant classification error at most $C_{err}$.
\begin{remark}
The intuition of using logistic loss is that logistic loss is closely connected to 0-1 loss. 
Generally,
logistic loss could be considered as a surrogate function of 0-1 loss. Moreover, under PAC halfspace learning setting, for any model $w$, its classification error could be bounded by a constant
times the population risk of its logistic loss.
\end{remark}
(2) With the pseudo labeler, next, we use a self-training algorithm STWN in \cite{frei2021self} (Algorithm \ref{alg3}) to convert the weak learner (pseudo labeler) to a strong learner. The self-training algorithm can ensure that, for data coming from an isotropic mixture distribution and if there is an initial pseudo labeler $w_{pl}$ that has small classification error, then the algorithms yield a  classifier with classification error arbitrarily close to the optimal one using only unlabeled examples. In each iteration of the STWN algorithm, we first use the pseudo labeler to label a batch of unlabeled data. Then we use the gradient descent with loss function $\tilde{\ell}$ on the pseudo labeled data to update the pseudo labeler. Note that the loss functions used in this self-training algorithm have to be "well-behaved", which is defined as follows:
\begin{definition}[Well behaved loss function \cite{frei2021self}]
	If the loss $\ell(z)$ is 1-Lipschitz, decreasing on the interval $[0,\infty)$ and for some constant $C_{\ell}\geq 1 $, $\ell^{'}(z)\geq \frac{1}{C_{\ell}}e^{-z}$ holds when $z>0$, then we say the loss function is well behaved.
\end{definition}

%\begin{remark} The loss functions we referred above are univariate functions, where $z$ is a univalent function of $w$ (for instance, $z=y \langle w,x\rangle$ ) as we will see in the algorithms later. Also, we should notice that the behaviors required on the loss functions are only mandatory on the interval $[0,\infty)$ rather than in the whole domain. This is because in the self-training algorithm we use, the input of the loss function is always non-negative.
% \end{remark}
Many loss functions are well behaved. For example, the exponential loss $\tilde{\ell}(z)=e^{-z}$ and the logistic loss $\tilde{\ell}(z)=\log(1+e^{-z})$ satisfies the above "well behaved" definition with $C_{\ell}=1$ and 2 respectively. In this paper, we will use the logistic function.

	\begin{algorithm}
	\caption{Self-training using weight normalization: STWN$(\{q_i\}_{i=1}^{T\times B},w_{pl})$ \label{alg3}}
	\begin{algorithmic}[1]
		\STATE {\bfseries Input:} The designed estimation error $\alpha$; parameters about the distribution: $K,U,r$; failure probability $\beta$; temperature $\sigma>0$, batch size $B$ and iteration $T$; $T\times B$ unlabeled public data $\{q_i\}_{i=1}^{T\times B}$; pseudo labeler $w_{pl}$.
		\STATE Set step size $\eta=\tilde{\Theta}\left(\frac{\alpha}{d(\log(\frac{1}{\beta}))^2}\right)$
		
		\STATE Let $w^{(0)}=\frac{w_{pl}}{||w_{pl}||_2}$
		
		\FOR {$t=0,\cdots,T-1$}
		\FOR {$i=1\cdots,B$}
		\STATE Generate pseudo labels $\hat{y}_{B\times t+i}=\text{sign}(\langle q_i,w^{(t)} \rangle)$
		\ENDFOR\\
		\STATE $v^{(t+1)}=w^{(t)}-\frac{\eta}{B}\sum\limits_{i=t\times B+1}^{B\times (t+1)}\nabla \tilde{\ell}\left(\frac{\hat{y}_i\cdot \langle q_i,w^{(t)}\rangle}{\sigma}\right)$
		\STATE $w^{(t+1)}=\frac{v^{(t+1)}}{||v^{(t+1)}||}$
		\ENDFOR\\
		\STATE Return $\{w^{(t)}\}_{t=0}^T$
	\end{algorithmic}
\end{algorithm}

The whole picture of our NLDP algorithm for mixture distributions is given in Algorithm \ref{alg5}, and its theoretical guarantee is provided by the following theorem:
\begin{theorem}\label{th4}
Assume that $(x,y)\sim \mathcal{P}$ follows a mixture distribution with $||\boldsymbol{\mu}||_2=\Theta(1)$ and known parameters $K,U,r=\Theta(1)$, and $\|x\|_2\leq R=O(1)$ for $x\sim \mathcal{P}_x$. Then if $\|\boldsymbol{\mu}\|_2\geq 3K\max\{\log \frac{8}{C_{err}}, 22K\}$, for any $\alpha, \beta, \epsilon, \delta\in (0, 1)$, there exist $w\in \{w^{(t)}\}_{t=0}^T$ which is $(\epsilon, \delta, \alpha, \beta)$-NLDP Learner with sample complexity  $m=\tilde{O}( \frac{d\log^3 \frac{1}{\delta} }{\alpha^2})$ for public unlabeled data and $n=\tilde{O}\left(d\text{Poly}(\log \frac{1}{\delta}, \frac{1}{\epsilon})\right)$ for private data, where the Big-$\tilde{O}$ omits other logarithmic terms.  

%$\tilde{\ell}$ is well behaved for some $C_{\tilde{\ell}}>1$. $\ell$ is smooth generalized linear loss and we assume $\ell{(z)}$ a decreasing function on $[0,\infty)$ in terms of $z$. We further assume that the temperature satisfies $\sigma\geq R\vee||\boldsymbol{\mu}||$ and $\underset{(x,y)\sim \mathcal{P}}{Pr}[y\langle w^*,x\rangle> 0]=1$. For $\alpha,\beta\in(0,1)$, $B=O\left(\frac{\log(\frac{1}{\beta})}{\alpha}\right)$, $T=\tilde{O}\left( \frac{d(\log(\frac{1}{\beta}))^2}{\alpha}  \right)$,  Algorithm \ref{alg5} is $(\epsilon,\delta)$ non-actively LDP using $T\times B=\tilde{O}\left(\frac{d(\log(\frac{1}{\beta}))^3}{\alpha^2}\right)$ public unlabeled data and $n=O\left(\frac{d}{\epsilon^{2+c}}\right)$ private data, and with probability at least $ 1- \beta$, there exists $t^{*}< T$ such that $err_{\mathcal{P}}(h_{w^{(t^{*})}})\leq err_{\mathcal{P}}(h_{\boldsymbol{\mu}})+ \alpha$, where $err_{\mathcal{P}}(h_{\boldsymbol{\mu}})$ is the error of Bayes-optimal classifier.
\end{theorem}
\begin{remark}
Although the general idea of Algorithm \ref{alg1}  and \ref{alg5} are almost the same, {\em i.e.,} use private data to build a weak learner or a pseudo labeler and use it to transform to a strong learner. There are still several critical differences. First, in Algorithm \ref{alg1} we need the weak learner $w^{priv}$ to have a constant classification error $\lambda<\frac{1}{2}$, while in Algorithm \ref{alg5} we aim to make the classification error of $w^{priv}$ be $C_{err}$ which needs to depend on the underlying distribution. Thus, we cannot use $w^{priv}$ in Algorithm \ref{alg1} to Algorithm \ref{alg5}. Second, the procedure of transforming is different, while in Algorithm \ref{alg5} the labeling is adaptive, Algorithm \ref{alg1} is non-adaptive. Thus, the idea of Algorithm \ref{alg5} is more similar to self-supervised learning and therefore needs less public data than Algorithm \ref{alg1}. Thirdly, while we can guarantee that the output of Algorithm \ref{alg1} is an NLDP learner, we can only ensure the existence of  NLDP learner among $\{w^{(t)}\}_{t=0}^T$  in  Algorithm \ref{alg5}. Finding out such a learner needs an additional one round. We leave it as an open problem for improving the algorithm. 
\end{remark}
\section{Conclusion}
We studied  the problem of PAC learning halfspaces in the non-interactive local differential privacy model (NLDP). Previous results either have  either exponential sample complexities or they need the large margin assumption of the data. Here we considered a relaxed setting where the server has access to some additional public but unlabeled data. Specifically, under different mild assumptions on the underlying data distribution, we proposed two approaches that are based on the Massart noise model and self-supervised learning and showed that it is possible to achieve sample complexities that are only linear in the dimension and polynomial in other terms for both private and public data, which significantly improve the previous results.
\section*{Acknowledgment}
Di Wang was support in part by the baseline funding BAS/1/1689-01-01, funding from the CRG grand URF/1/4663-01-01, FCC/1/1976-49-01 from CBRC and funding from the AI Initiative REI/1/4811-10-01 of King Abdullah University of Science and Technology (KAUST).  He was also supported by the funding of the SDAIA-KAUST Center of Excellence in Data Science and Artificial Intelligence (SDAIA-KAUST AI).

\bibliography{acml22}

\clearpage
\onecolumn
\appendix 
\section{Omitted Proofs}
\subsection{ Proof of Theorem \ref{th2}}
\begin{proof}
Before we start our proof, we give the definition of $LR_S$ oracle and $\text{STAT}_{\mathcal{P}}(\tau)$ oracle to prepare the readers for the proof.
$LR_S$ oracle is based on the local randomizer which is defined as follows:
\begin{definition} ($\epsilon$-local randomizer)
An $\epsilon$-local randomizer $\mathcal{R}: Z \rightarrow W$ is a randomized algorithm that $\forall z_1,z_2\in Z$ and $\forall w \in W$, it satisfies:
$$ Pr[\mathcal{R}(z_1)=w]\leq e^{\epsilon}[\mathcal{R}(z_2)=w]$$.
\end{definition}
\begin{definition}($LR_S$ oracle \cite{kasiviswanathan2011can})
For a dataset $S\in Z^n$, an $LR_S$ oracle takes an index $i$ and a local randomizer $\mathcal{R}$ as inputs and outputs a random value $w$ obtained by applying $\mathcal{R}(z_i)$. 
\end{definition}
And we recall the definition of statistical queries. 
\begin{definition}
    Let $\mathcal{P}$ be an distribution over a domain $Z$ and $\tau>0$. A statistical query oracle $\text{STAT}_{\mathcal{P}}(\tau)$ is an oracle that given any function $\phi:Z\rightarrow [-1,1]$ as input, the statistical query oracle returns some value $v$ such that $|v-\mathbb{E}_{z\sim\mathcal{P}}[\phi(z)]|\leq \tau$.
\end{definition}
Now we formally begin our proof.
First, we prove that the algorithm given in \cite{daniely2019locally} uses the same number of private data and public data. The core idea of the algorithm in \cite{daniely2019locally} is that: when using the projected gradient descent to find a vector $w$ that satisfies 
$\underset{(x,y)\sim \mathcal{P}}{Pr}[y\neq \text{sign}(\langle \hat{w},x\rangle )]\leq \alpha$,  the objective function can be decomposed as $F(w)=F_1(w)+F_2(w)$, where the (sub-)gradient of $F_1(w)$ (namely $\nabla F_1(w)$) is just a function of $x$ while 
the gradient of $F_2(w)$ (namely $\nabla F_2(w)$) is independent of $w$. 
As a result, (sub)-gradient $\nabla F(w)$ can be computed non-interactively by calculating $\nabla F_1(w)$ with only public unlabeled data and calculating $\nabla F_2(w)$ with non-interactive statistic queries because $\nabla F_2(w)$ doesn't depend on $w$. So to make this algorithm achieve the PAC learning error $\alpha$, the sample complexity of the private data and the public data should be the same. For more details, please refer to the proof of Lemma 4.3 in \cite{daniely2019locally}. So, to prove our theorem, we only have to prove that the sample complexity of the private data is $\tilde{O}(\frac{d^{10}\log(1/\beta)}{\epsilon^2\cdot\gamma^{12}\alpha^6})$.

In the following, we give the private sample complexity of the algorithm in \cite{daniely2019locally}, which can be directly derived from the following two lemmas.

The first Lemma states that a statistic query oracle $\text{STAT}_{\mathcal{P}}(\tau)$ can be simulated with success probability $1-\beta$ by $\epsilon$-LDP algorithm using $LR_{S}$ oracle.
\begin{lemma}\cite{kasiviswanathan2011can}\label{le12}
Let $\mathcal{A}_{SQ}$ be an algorithm that makes at most $t$ queries to $\text{STAT}_{\mathcal{P}}(\tau)$ oracle. Then for any $\epsilon>0$ and $\beta>0$, there is an $\epsilon$-LDP algorithm $\mathcal{A}_{priv}$ that uses $LR_s$ oracle for $S$ containing $n=O(\frac{t\log(\frac{t}{\beta})}{(\epsilon\tau)^2})$ i.i.d. samples from $\mathcal{P}$ and produces the same output as $\mathcal{A}_{SQ}$ with probability at least $1-\beta$. Further, if $\mathcal{A}_{SQ}$ is non-interactive then $\mathcal{A}_{priv}$ is non-interactive.
\end{lemma}
The next lemma claims the existence a NLDP algorithm $\mathcal{A}_{SQ}$ that achieves PAC learning error $\alpha$ for any arbitrary $\alpha\in(0,1)$.
\begin{lemma}[Lemma 4.3 in \cite{daniely2019locally}]\label{le13}
Let $\mathcal{P}$ be a distribution on $\mathcal{B}_2^d\times\{\pm 1\}$ such that there is a vector $w^{*}\in \mathcal{B}_2^d$ satisfying $Pr_{(x,y)\sim \mathcal{P}}[y\langle w^{*},x\rangle \geq \gamma]=1$. Then there is a non-interactive algorithm $\mathcal{A}_{SQ}$ that for every $\alpha\in (0,1)$, it uses $O(\frac{d^4}{\gamma^4\alpha^2})$ queries to $\text{STAT}_{\mathcal{P}}(\Omega(\frac{\gamma^4\alpha^2}{d^3}))$ and finds a vector $\hat{w}$ such that $\underset{(x,y)\sim \mathcal{P}}{Pr}[y\neq \text{sign}(\langle \hat{w},x\rangle )]\leq \alpha$.
\end{lemma}
Lemma \ref{le13} indicates that if we can find a non-interactive algorithm $\mathcal{A}_{SQ}$ that makes at most $t$ queries to $\text{STAT}_{\mathcal{P}}(\tau)$ oracle, then with probability $1-\beta$, the existence of an $\epsilon$-NLDP algorithm $\mathcal{A}_{priv}$ is guaranteed using $n=O(\frac{t\log(\frac{t}{\beta})}{(\epsilon\tau)^2})$
private data. 
 So, by substituting $t=O(\frac{d^4}{\gamma^4\alpha^2})$ and $\tau=\Omega(\frac{\gamma^4\alpha^2}{d^3})$ in Lemma \ref{le12}, the sample complexity of public data is straight forward.

\end{proof}

\subsection{{\bf Proof of Lemma \ref{le7}}}
\begin{proof}
To proof Lemma \ref{le7}, we first study the excess empirical risk with the hinge loss $\ell(w, (x, y))=\max\{0, 1-y\langle w, x\rangle \}$ of the output $w_t$ of the algorithm $\mathcal{H}_{priv}(\frac{1}{32R}, \epsilon, \delta, \tilde{S}_t)$. First, we recall the following result of  $\mathcal{H}_{priv}(\alpha, \epsilon, \delta, S)$ if each $\|x_i\|_2\leq 1, |y_i|\leq 1$. 
\begin{lemma}[Theorem 30 in \cite{wang2020empirical}]\label{lemma:6}
For any $0<\epsilon, \delta<1$,   if each $\|x_i\|_2\leq 1, |y_i|\leq 1$ for all $i\in [n]$, $\mathcal{H}_{priv}(\alpha, \epsilon, \delta, S)$ is $(\epsilon, \delta)$-NLDP. Moreover, for any error $\alpha \in (0, 1)$, if the size of dataset $n$ is sufficiently large such that $n\geq \tilde{\Omega}(\frac{C^pp^{6p}d}{\epsilon^{4p+4}\alpha})$ with $p=O(\frac{1}{\alpha^3})$. Then the output $w_n$ satisfies 
	\begin{equation}
	\mathbb{E} [\frac{1}{n}\sum_{i=1}^n \max\{0, \frac{1}{R}-y\langle w, x\rangle \}]-\min_{||w||_2\leq 1}\frac{1}{n}\sum_{i=1}^n \max\{0, \frac{1}{R}-y\langle w, x\rangle \}\leq \alpha,
	\end{equation}
	where  $C>0$ is a constant\footnote{Note that \cite{wang2020empirical} only showed the case where $R=2$. However, it is obvious to extend to the general $R$ with the same proof.}  and the expectation is taken over the internal randomness of the algorithm. 
\end{lemma}
Note that in we need to assume $\|x_i\|_2\leq 1$ in Lemma \ref{lemma:6} while in our setting $\|x_i\|_2\leq R$. Thus, we need to normalize the data to $\tilde{S}_t$ first and revoke $\mathcal{H}_{priv}(\frac{1}{32R}, \epsilon, \delta, \tilde{S}_t)$. By Lemma \ref{lemma:6} we have when $\frac{n}{k}\geq \tilde{\Omega}(d\text{Poly}(\frac{1}{\epsilon}, \log \frac{1}{\delta}) )$
	\begin{equation}\label{aeq:2}
	\mathbb{E} [\hat{L}(w_t,\tilde{S}_t)]-\min_{||w||_2\leq 1}\hat{L}(w,\tilde{S}_t)\leq \frac{1}{32R},
	\end{equation}
	where $\hat{L}(w_t,\tilde{S}_t)=\frac{1}{|S_t|}\sum_{(x_i, y_i)\in S_t}\max\{0, \frac{1}{R}-y_i\langle w, \frac{x_i}{R}\rangle \}$.  Thus, we have the following result via multiplying $R$ in both side of (\ref{aeq:2}). 
\begin{lemma}\label{le2}
When $n\geq \tilde{\Omega}(dk\text{Poly}(\frac{1}{\epsilon}, \log \frac{1}{\delta}) )$, each $w_t=\mathcal{H}_{priv}(\frac{1}{32R}, \epsilon, \delta, \tilde{S}_t)$ for $t\in [k]$ satisfies 
	\begin{equation}\label{eq1}
	\mathbb{E} [\hat{L}(w_t, S_t)]-\min_{||w||_2\leq 1}\hat{L}(w_t ,S_t)\leq \frac{1}{32},
	\end{equation}
	where $\hat{L}(w_t, S_t)$ is the empirical risk of $\ell(w, (x, y))=\max\{0, 1-y\langle w, x\rangle \}$, 
and the expectation is taken over  the internal randomness of the algorithm. 
	\end{lemma}
The following lemma transforms the excess empirical risk in Lemma \ref{le2} to classification error. 

\begin{lemma}\label{le3}
Under the assumptions in Theorem \ref{th1},  then for any $t \in [k]$, $\beta\in (0,1)$, with probability at least $1-\frac{\beta}{2}$, the following holds 	when $n\geq \tilde{\Omega}(dk\text{Poly}(\frac{1}{\epsilon}, \log \frac{1}{\delta}) )$ with $k=O(\log \frac{1}{\beta})$. 
	\begin{equation*}
	\mathbb{E}[err_{P}(h_{w_t})]\leq\frac{1}{8}
	\end{equation*}
where the expectation is taken over the random choice of the data in $D$ and the internal randomness of $\mathcal{H}_{priv}$. 
\end{lemma}
\begin{proof}[{\bf Proof of Lemma \ref{le3}}]
We need the following lemma for our proof.
\begin{lemma}[\cite{anthony2009neural}]\label{le1}
	Let $\mathcal{H}$ be the set of $\{\pm 1\}$-valued functions defined on a set $\mathcal{X}$ and $\mathcal{P}$ is a probability distribution on $Z = \mathcal{X}\times \{\pm 1\}$. For $\eta \in (0,1)$,  $\zeta>0$, $Pr_{z\sim \mathcal{P}^n}[\exists h \in \mathcal{H}: err_{\mathcal{P}}(h)>(1+\zeta)err_z(h)+\eta] \leq 4 \tau_{\mathcal{H}}(2n) e^{-\frac{\eta \zeta n}{4(\zeta +1)}}$, where $err_{\mathcal{P}}(h)$ is the population error, $err_z(h)$ is the empirical error on sample set $z$ and $\tau_{\mathcal{H}}(\cdot)$ is the growth function of $\mathcal{H}$.
	If $\mathcal{H}$ is the hypothesis set of learning halfspaces, then $\tau_{\mathcal{H}}(2n)\leq (2n)^{d+1}+1$ with $d$ being the dimension of set $\mathcal{X}$.
	\end{lemma}
	The following proof applies for any $t\in [k]$:
	
	Based on our assumption, the halfspace is separable, so we know that $\min_{||w||_2\leq 1}\hat{L}(w,S_t)=0$.
Since hinge loss is a convex surrogate for $0-1$ loss, we can get that 

$\mathbb{E} [err_{S_t}(h_{w_t})]\leq \mathbb{E}[\hat{L}(w_t,S_t)]{\leq} \min_{||w||_2\leq 1}\hat{L}(w,D) +\frac{1}{32}=\frac{1}{32} $, where the second inequality comes from (\ref{eq1}).

Setting $\eta=\frac{1}{16}$ and $\zeta=1$, for any $t\in [k]$, denoting $n_t=|S_t|$, then according to Lemma \ref{le1}, we can get 
\begin{equation}\label{eq2}
\underset{S_t\sim \mathcal{P}^{n_t}}{Pr}\{\exists h_{w_t} \in \mathcal{H}: \mathbb{E}[err_{\mathcal{P}}(h_{w_t})]>2\cdot \frac{1}{32}+\frac{1}{16}\}\leq 4\tau_{\mathcal{H}}(2n_t)e^{-\frac{n_t}{128}}.
\end{equation}
When $n_t=\tilde{\Omega}({d\log\frac{1}{\beta}}\text{Poly}(\log \frac{1}{\delta}, \frac{1}{\epsilon}))$,  we have $4\tau_{\mathcal{H}}(2n_t)e^{-\frac{n_t}{128}}\leq \frac{\beta}{2k}$.
Then (\ref{eq2}) will become 
\begin{equation*}
\underset{S_t\sim \mathcal{P}^{n_t}}{Pr}\{\exists h_{w_t} \in \mathcal{H}: \mathbb{E}[err_{\mathcal{P}}(h_{w_t})]>\frac{1}{8}\}\leq\frac{\beta}{2k}.
\end{equation*}
Thus, take the union bound, we have with probability at least $1-\frac{\beta}{2}$  for any $t\in [k]$,
\begin{equation*}
\mathbb{E}[err_{\mathcal{P}}(h_{w_t})]\leq\frac{1}{8}. 
\end{equation*}
\end{proof}
According to Lemma \ref{le3}, for any $t\in [k]$, with probability at least $1-\frac{\beta}{2}$,
we  have 
\begin{equation*}
\mathbb{E}_{D,\mathcal{H}_{priv}}[err_{\mathcal{P}}(h_{w_t})]=\mathbb{E}_{D,\mathcal{H}_{priv}}\{\underset {(x,y)\sim \mathcal{P}}{Pr}[h_{w_t}(x)\neq y]\}\leq \frac{1}{8}.
\end{equation*}
%Since $\frac{1}{8}<\frac{1}{2}$, it holds that 
%\begin{equation*}
%\underset{(x,y)\sim \mathcal{P}}{Pr}[\hat{f}(x)\neq y]\leq\underset{(x,y)\sim %\mathcal{P}}{Pr}[h_{w_t}(x)\neq y]\leq \frac{1}{8}
%\end{equation*}

Applying Hoeffding inequality, we have 

\begin{align*}
{Pr}\{\underset{(x,y)\sim \mathcal{P}}{Pr}[\hat{f}(x)\neq y]-\frac{1}{8}>\frac{1}{4}\}& \leq Pr\{ \frac{1}{k}\sum\limits_{t=1}^k \underset{(x,y)\sim \mathcal{P}}{Pr}[h_{w_t}(x)\neq y]-\frac{1}{8}>\frac{1}{16}\}
\\&\leq Pr\{ |\frac{1}{k}\sum\limits_{t=1}^k \underset{(x,y)\sim \mathcal{P}}{Pr}[h_{w_t}(x)\neq y]-\frac{1}{8}|>\frac{1}{16}\}\leq 2e^{-\frac{k}{32}}
\end{align*}
For the first inequality, denote the event $E_1=\{\underset{(x,y)\sim \mathcal{P}}{Pr}[\hat{f}(x)\neq y]-\frac{1}{8}>\frac{1}{4}\}$ and event $E_2=\{ \frac{1}{k}\sum\limits_{t=1}^k \underset{(x,y)\sim \mathcal{P}}{Pr}[h_{w_t}(x)\neq y]-\frac{1}{8}>\frac{1}{16}\}$. Thus, the first inequality holds if $E_1\subseteq E_2$. $E_1$ claims that with probability at least $\frac{3}{8}$ the classifier $\hat{f}$ will gives wrong prediction. That is more than half of $\{w_t\}_{t=1}^k$ give wrong predictions. Thus, $\frac{1}{k}\sum\limits_{t=1}^k \underset{(x,y)\sim \mathcal{P}}{Pr}[h_{w_t}(x)\neq y]\geq \frac{\frac{k}{2}\times \frac{3}{8}}{k}=\frac{3}{16}$. 
The second inequality is due to $\mathbb{E}\{\frac{1}{k}\sum\limits_{t=1}^k \underset{(x,y)\sim \mathcal{P}}{Pr}[h_{w_t}(x)\neq y]\}\leq \frac{1}{8}$.

When $k=O(\log(\frac{1}{\beta}))$,
we have 
\begin{equation*}
{Pr}\{\underset{(x,y)\sim \mathcal{P}}{Pr}[\hat{f}(x)\neq y]>\frac{3}{16}\}\leq \frac{\beta}{2}
\end{equation*}
Therefore, with probability at least $1-\frac{\beta}{2}-\frac{\beta}{2}=1-\beta$, we have 
\begin{equation*}
\underset{(x,y)\sim \mathcal{P}}{Pr}[\hat{f}(x)\neq y]\leq \frac{3}{16}
\end{equation*}
 \end{proof}

\subsection{\bf Proof of Theorem \ref{th4}}
The proof of this theorem can be induced directly by the following two lemmas. The first lemma claims that Logistic Loss-NLDP  outputs a classifier $w^{priv}$ which is NLP and achieves a constant classification error $C_{err}$ using $O(d\text{Poly}(\frac{1}{\epsilon}))$ private samples.

\begin{lemma}\label{le5}
Algorithm \ref{alg5} is $(\epsilon,\delta)$-NLDP and $w^{priv}$ satisfies the following when $n=O(d\text{Poly}(\frac{1}{\epsilon}))$
 $$err_{\mathcal{P}}(h_{w^{priv}})\leq \frac{r^2}{144U}.$$ 
\end{lemma}
The second lemma claims that STWN (Algorithm \ref{alg3}) transforms a weak learner that achieves a constant classification error to a strong learner that achieves a classification error arbitrarily close to the Bayes-optimal error using only unlabeled samples. 
\begin{lemma}\cite{frei2021self}\label{le4}
If $(x,y)\sim \mathcal{P}$ is a mixture distribution with mean $\boldsymbol{\mu}$ satisfying $||\boldsymbol{\mu}||_2=\Theta(1)$ and $K,U,r>0$, assume $\tilde{\ell}$ is well behaved for some $C_{\tilde{\ell}}\geq 1$ and the temperature satisfies $\sigma\geq R\vee ||\boldsymbol{\mu}||_2$. Assume access to a pseudo labeler $w_{pl}$ which achieves classification error less than $\frac{R^2}{72C_{\tilde{\ell}}U}$, i.e., $err_{\mathcal{P}}(h_{w_{pl}})\leq \frac{R^2}{72C_{\tilde{\ell}}U}$. Let $\alpha,\beta\in(0,1)$, $B=\Omega\left(\frac{\log(\frac{1}{\beta})}{\alpha}\right)$, $T=\tilde{\Omega}\left( \frac{d^2(\log(\frac{1}{\beta}))}{\alpha}  \right)$ and step size $\eta=\tilde{\Theta}\left(\frac{\alpha}{d(\log(\frac{1}{\beta}))^2}\right)$, running STWN (Algorithm \ref{alg3}) with $T\times B$ unlabeled samples, then with probability at least $1-\beta$, there exists $t^{*}< T$ such that $err_{\mathcal{P}}(h_{w^{(t^{*})}})\leq err_{\mathcal{P}}(h_{\boldsymbol{\mu}})+\alpha$ where $err_{\mathcal{P}}(h_{\boldsymbol{\mu}})$ is the error of Bayes-optimal classifier.

In particular, let $B=O\left(\frac{\log(\frac{1}{\beta})}{\alpha}\right)$, $T=\tilde{O}\left( \frac{d(\log(\frac{1}{\beta}))^2}{\alpha}  \right)$, above conclusion holds using $T\times B=\tilde{O}\left(\frac{d(\log(\frac{1}{\beta}))^3}{\alpha}\right)$ unlabeled data samples.
\end{lemma}
\paragraph{\bf Proof of Theorem \ref{th4}:} Since in Algorithm \ref{alg5} we use the logistic function as the well behaved loss, we have $C_{\tilde{\ell}}=2$. Moreover, under our assumption, the Bayes-optimal classifier is just $w^*$ and thus $err_{\mathcal{P}}(h_{\boldsymbol{\mu}})=0$. Combing with Lemma \ref{le5} and Lemma \ref{le4} we finish the proof. 
\iffalse 
\begin{remark}
Notice that logistic loss $\ell(z)=\log (1+e^{-z})$ belongs to the smooth generalized linear loss family and it is a decreasing function on $[0,\infty)$. Also, since logistic loss is "well-behaved", both $\ell$ and $\tilde{\ell}$ can be substituted by logistic loss. 
\end{remark}

\fi

\begin{proof}[{\bf Proof of Lemma \ref{le5}}]
	To prove the lemma, we need the following lemma claiming the excess population loss of the output of Logistic Loss-NLDP: $\mathcal{T}_{priv}(\alpha, R, \epsilon, \delta, D)$ 
\begin{lemma}[Theorem 6 in \cite{zheng2017collect}]\label{le9}
For any $0<\epsilon, \delta\leq 1$, if each $\|x_i\|_2\leq R$ and $y\in \{-1, 1\}$ for all $i\in [n]$, and $\mathcal{W}=\{w: \|w\|_2\leq \rho\}$, $\mathcal{T}_{priv}(\alpha, R, \rho, \epsilon, \delta, D)$ is $(\epsilon, \delta)$-NLDP. Moreover, for any given  error $\alpha \in (0, 1)$, if the size of dataset $n$ is sufficiently large such that $$n\geq \tilde{\Omega}\left(\left(\frac{8R\rho}{\alpha}\right)^{4R\rho\ln\ln \frac{8R\rho}{\alpha}}\left(\frac{4R\rho}{\epsilon}\right)^{2cR\rho\ln \frac{8R\rho}{\alpha}+2}\frac{1}{\alpha^2\epsilon^2} \right).$$ Then the output $w_n$ satisfies $\mathbb{E}[L(w^{priv})]-\min_{w\in \mathcal{W}}L(w)\leq \alpha$, where $L(w^{priv})$ is the population risk of the logistic loss, i.e., $L(w)=\mathbb{E}_{(x,y)\sim \mathcal{P}}[\ell(w;x,y)]$, where $\ell(w;x,y)=\log(1+e^{-y\langle x, w\rangle})$.
\end{lemma}

Apply the above Lemma \ref{le9} with $\alpha=\frac{C_{err}\log 2}{2}=\frac{\log 2r^2}{144U}$ and $\rho=\|\boldsymbol{\mu}\|_2$. Then using $n=O\left(d\text{Poly}(\frac{1}{\epsilon}) \right)$  private samples, $w^{priv}$ achieves the excess population loss no more than $\frac{\log 2 r^2}{144U}$, i.e., $  \mathbb{E}[{L}(w^{priv})]- \min_{\|w\|_2\leq \|\boldsymbol{\mu}\|_2} \frac{\mathbb{E}[{L}(w)]\leq C_{err} \log 2}{2}$. Since $\|\boldsymbol{\mu}\|_2\in \mathcal{W}$, thus, 
\begin{equation*}
     \mathbb{E}[{L}(w^{priv})]\leq \mathbb{E}[{L}(\boldsymbol{\mu})]+ \frac{C_{err}\log 2}{2}. 
\end{equation*}
For the term of $\mathbb{E}[{L}(\boldsymbol{\mu})]$, recall the following lemma. 
\begin{lemma}[Lemma B.3 in \cite{frei2021self}]
Consider the logistic function $\ell(z)=\log (1+e^{-z})$. Let $(x, y)\sim \mathcal{P}$ be a mixture distribution with mean $\boldsymbol{\mu}$ and parameters $K, U, R=\Theta(1)>0$. Then if $\|\boldsymbol{\mu}\|_2\geq 64K^2$ we have 
\begin{equation}
    \mathbb{E}_{(x,y)\sim \mathcal{P}}[\ell(y\langle w, x\rangle)]\leq \exp(-\frac{\|\boldsymbol{\mu}\|_2}{3K}). 
\end{equation}
\end{lemma}
By using the previous lemma, we have 
\begin{equation*}
     \mathbb{E}[{L}(w^{priv})]\leq \exp(-\frac{\|\boldsymbol{\mu}\|_2}{3K})+ C_{err}\log 2/2 \leq C_{err}\log 2, 
\end{equation*}
where the last inequality is due to the assumption of $\|\boldsymbol{\mu}\|_2\geq 3K\log (8/C_{err})$. Thus we have 

\begin{align*}
&Pr[y\neq \text{sign}(\langle w^{priv},x\rangle)]=Pr[y\cdot\langle w^{priv},x\rangle<0]=Pr[\ell(y\cdot\langle w^{priv},x\rangle)>\ell(0)]\\&\leq \frac{\mathbb{E}[\ell(y\cdot\langle w^{priv},x\rangle)]}{\ell{(0)}}
=  \frac{\mathbb{E}[L(w^{priv})]}{\ell(0)}\leq \frac{ r^2}{144U}
\end{align*}
where we use the monotonicity of the loss function and Markov's inequality.
\end{proof}

\section{Details of Hinge Loss-LDP and Logistic Loss-NLDP}
	\begin{algorithm}[!htbp]
 	\caption{Hinge Loss-NLDP: $\mathcal{H}_{priv}(\alpha, \epsilon, \delta, S)$}\label{alg11}
 		\textbf{Input}: 
		Private data $S=\{(x_i,y_i)\}_{i=1}^n\in \mathbb{R}^d\times \{\pm 1\}$, where $||x_i||_2\leq 1,||y_i||_2\leq 1$; Privacy parameters $\epsilon,\delta$; Error $\alpha$.
	\begin{algorithmic}[1]
			\STATE 	Denote $P_{p}(x)=\sum_{j=0}^{p} c_{i}\tbinom{p}{j} x^{j}(1-x)^{p-j}$ as the $p$-th order Bernstein polynomial for the function $f^{'}_{\beta}$, where $c_i=f_{\beta}^{'}\left(\frac{i}{p}\right)$ and $f_{\beta}(x)=\frac{\frac{1}{R}-x+\sqrt{(\frac{1}{R}-x)^2+\beta^2}}{2}$ with $\beta=\frac{\alpha}{4}$ and $p=\frac{2}{\beta^2\alpha}$.\\ 
			$\backslash\backslash$ The local user side:\\
		\FOR{$i\in [n]$}
			 \STATE Set $\sigma_{i,0}\sim \mathcal{N}\left(0,\frac{32\log(1.25/\delta)}{\epsilon^2}\boldsymbol{I}_d\right)$ and $z_{i,0}\sim \mathcal{N}\left(0,\frac{32\log(1.25/\delta)}{\epsilon^2}\right)$ 
			 \STATE Set 
			$x_{i,0}=x_i+\sigma_{i,0}$ and $y_{i,0}=y_i+z_{i,0}$
			\FOR{$j\in [p(p+1)]$}
	\STATE  $x_{i,j}=x_i+\sigma_{i,j}$,  where
	$\sigma_{i,j}\sim \mathcal{N}\left(0,\frac{8\log(1.25/\delta) p^2(p+1)^2}{\epsilon^2}\boldsymbol{I}_d\right)$
	\STATE 
	$y_{i,j}=y_i+z_{i,j}$, where
	$z_{i,j}\sim \mathcal{N}\left(0,\frac{8\log(1.25/\delta) p^2(p+1)^2}{\epsilon^2}\right)$
	\ENDFOR
\STATE	Send $\{x_{i,j}\}_{j=0}^{p(p+1)}$ and $\{y_{i,j}\}_{j=0}^{p(p+1)}$ to the server.
		\ENDFOR
	\end{algorithmic}
		$\backslash\backslash$ The server side:
	\begin{algorithmic}[1]
		\FOR{$t\in [n]$}
		\STATE Randomly sample $i\in [n]$ uniformly and set $t_{i,0}=1$
		\FOR{j=$\{0\}\cup [p]$}
		\STATE $t_{i,j}=\Pi_{k=jp+1}^{jp+j}y_{i,k}\langle w_t,x_{i,k}\rangle$ and $t_{i,0}=1$
		\STATE $s_{i,j}=\Pi_{k=jp+j+1}^{jp+p}(1-y_{i,k}\langle w_t, x_{i,k}\rangle )$ and $s_{i,p}=1$
		\ENDFOR
\STATE 		Denote $G(w_t,i)=(\sum_{j=0}^p c_j\tbinom{p}{j}t_{i,j}s_{i,j})y_{i,0}x_{i,0}^T$
\STATE Update SIGM (Algorithm \ref{alg9}) by $G(w_t,i)$
		\ENDFOR
		\\
		\STATE Return $w_n$
	\end{algorithmic}
\end{algorithm}

	\begin{algorithm}
	\caption{Logistic Loss-NLDP: $\mathcal{T}_{priv}(\alpha, R, \rho, \epsilon, \delta, D)$  \label{alg12}
	}
			\textbf{Input}: 	Private data $S=\{(x_i,y_i)\}_{i=1}^n\in \mathbb{R}^d\times \{\pm 1\}$, where $||x_i||_2\leq R,||y_i||_2\leq 1$; Privacy parameters $\epsilon,\delta$; Error $\alpha$; Constraint set $\mathcal{W}=\{w: \|w\|_2\leq \rho\}$. 
	 	\begin{algorithmic}[1]
	 	\STATE Denote the logistic loss with scale $R\rho$: $\ell(w, x, y, R)=\log(1+e^{-R\rho y\langle w, x\rangle})=-yh_1(R\rho w^Tx)+h_2(R\rho w^Tx)$, where $h_1(z)=\frac{z}{2}$ and $h_2(z)=\frac{z}{2}+\log(1+e^{-z})$. For the function $h'_1(R\rho\cdot): [-1, 1]\mapsto \mathbb{R}$ and $h'_2(R\rho\cdot) : [-1, 1]\mapsto \mathbb{R}$, denote the Chebyshev polynomial with degree $p$ for function $h'_1(R\rho\cdot)$ and $h'_2(R\rho\cdot)$ as $\sum_{i=1}^n c_{1k}x^k$ and $\sum_{i=1}^n c_{2k}x^k$ respectively, where the degree $p=O(R\ln\frac{R\rho}{\alpha})$.   \\ 
	 $\backslash\backslash$ The local user side:\\
	\FOR{$i\in [n]$}
	\STATE Normalize the data $x'_i=\frac{x_i}{R}$. 
	\STATE Set $\sigma_{i,0}\sim \mathcal{N}\left(0,\frac{32\log(1.25/\delta)}{\epsilon^2}\boldsymbol{I}_d\right)$ and $z_{i,0}\sim \mathcal{N}\left(0,\frac{32\log(1.25/\delta)}{\epsilon^2}\right)$ 
			 \STATE Set 
			$x_{i,0}=x'_i+\sigma_{i,0}$ and $y_{i,0}=y_i+z_{i,0}$
			\FOR{$j\in [p(p+1)]$}
	\STATE  $x_{i,j}=x'_i+\sigma_{i,j}$,  where
	$\sigma_{i,j}\sim \mathcal{N}\left(0,\frac{8\log(1.25/\delta) p^2(p+1)^2}{\epsilon^2}\boldsymbol{I}_d\right)$

	\ENDFOR
			\FOR{$j=p$}
	\STATE 
	$y_{i,j}=y_i+z_{i,j}$, where
	$z_{i,j}\sim \mathcal{N}\left(0,\frac{8\log(1.25/\delta) p^2}{\epsilon^2}\right)$
	\ENDFOR
	\STATE	Send $\{x_{i,j}\}_{j=0}^{p(p+1)}$ and $\{y_{i,j}\}_{j=0}^{p}$ to the server.
	\ENDFOR
	\end{algorithmic}
		$\backslash\backslash$ The server side:
	\begin{algorithmic}[1]
			\FOR{$t\in [n]$}
		\STATE Randomly sample $i\in [n]$ uniformly and set $t_{i,0}=1$
		\FOR{j=$\{0\}\cup [p]$}
			\STATE $t_j=\Pi_{k=\frac{j(j-1)}{2}+1}^{\frac{j(j+1)}{2}}(w_t^Tx_{i, k})$
		\ENDFOR 
		\STATE $\tilde{G}(w_t;i)=\left(\sum\limits_{k=0}^{p}(c_{2k}-c_{1k}y_{i,j})t_k (R\rho)^{k+1}
		\right)z_0$. 
		\STATE Update SIGM (Algorithm \ref{alg9}) by $\tilde{G}(w_t;i)$ to obtain $w_{t+1}$. 
				\ENDFOR  
	    \STATE Return $w_{n+1}$
	    
	\end{algorithmic}
\end{algorithm}

	\begin{algorithm}
	\caption{Stochastic Intermediate Gradient Method (SIGM)
	%	\cite{dvurechensky2016stochastic}
	}
	\label{alg9}
	\textbf{Input:}
		The sequences $\{\alpha_i\}_{i\geq 0}$, $\{\beta_i\}_{i\geq 0}$, $\{B_i\}_{i\geq 0}$ functions $d(x)=\frac{||x||^2}{2}$, Bregman distance $V(x,z)=d(X)-d(Z)-\langle \nabla d(z),x-z\rangle$.
		
	\begin{algorithmic}[1]
		\STATE Compute $x_0=\arg \min _{x\in \mathcal{C}}\{d(x)\}$.
		\STATE Let $\xi_0$ be a realization of the random variable $\xi$.
		\STATE Compute $y_0=\arg\min_{x\in \mathcal{C}}\{\beta_0d(x)+\alpha_0\langle G_{\gamma,\beta,\sigma}(x_0; \xi_0),x-x_0\rangle\}$
		\FOR{$k\in \{0\}\cup [T-1]$}
		\STATE Compute $z_k=\arg\min_{x\in \mathcal{C}}\{\beta_kd(x)+\sum_{i=0}^k\alpha_i\langle G_{\gamma,\beta,\sigma}(x_i; \xi_i),x-x_i\rangle\}$
		\STATE Let $x_{k+1}=\eta_kz_k+(1-\eta_k)y_k$
		\STATE Let $\xi_{k+1}$ be a realization of the random variable 
		$\xi$
		\STATE Compute $\hat{x}_{k+1}=\arg\min_{x\in \mathcal{C}}\{\beta_kV(x,z_k)+\alpha_{k+1}\langle G_{\gamma,\beta,\sigma}(x_{k+1}; \xi_{k+1}),x-z_k\rangle\}$
		\STATE Let $w_{k+1}=\eta \hat{x}_{k+1}+(1-\eta_k)y_k$
		\STATE $y_{k+1}=
		\frac{A_{k+1}-B_{k+1}}{A_{k+1}}y_k+\frac{B_{k+1}}{A_{k+1}}w_{k+1}$
		\ENDFOR
		\\
		\STATE Return $y_T$
	\end{algorithmic}
\end{algorithm}

\iffalse 
	\begin{algorithm}
	\caption{Basic Private Vector mechanism: BPV($\boldsymbol{x},\epsilon,\delta$)}\label{alg4}
	\begin{algorithmic}[1]
		\STATE {\bfseries Input:}
	A vector $\boldsymbol{x}\in \mathbb{R}^d$, private parameters $\epsilon, \delta$.
	\STATE Set $\sigma=\frac{\sqrt{2\log(1.25/\delta)}}{\epsilon}$
	\IF {$||\boldsymbol{x}||_2>1$}
	
	\STATE{$\boldsymbol{x}=\frac{\boldsymbol{x}}{||\boldsymbol{x}||_2}$}	
	\ENDIF\\
	\STATE $\boldsymbol{z}=\boldsymbol{x}+\boldsymbol{\xi}$, where $\boldsymbol{\xi}\sim \mathcal{N}(0,\sigma^2\mathbb{I}_d)$\\
		\STATE Return 
		Private vector $\boldsymbol{z}$
	\end{algorithmic}
\end{algorithm}
\fi 


%

\end{document}